\documentclass[conference]{IEEEtran}
\usepackage{cite}
\usepackage{url}
\usepackage{amsmath,amsthm,amssymb,amsfonts}
\usepackage{graphicx}
\usepackage{textcomp}
\usepackage{xcolor}
\usepackage{fancyhdr}
\usepackage{algorithm}
\usepackage{algorithmic}
\usepackage{todonotes}
\usepackage{siunitx}
\usepackage{enumitem}
\usepackage{makecell}

\usepackage[utf8]{inputenc} % allow utf-8 input
\usepackage[T1]{fontenc}    % use 8-bit T1 fonts
\usepackage{hyperref}       % hyperlinks
\usepackage{url}            % simple URL typesetting
\usepackage{booktabs}       % professional-quality tables
\usepackage{nicefrac}       % compact symbols for 1/2, etc.
\usepackage{microtype}      % microtypography

\usepackage{xspace}

     \def\RR{\mathbb{R}}

 \newcommand{\Reg}{\emph{Reg}\xspace}
 \newcommand{\Loc}{\emph{Loc}\xspace}

\newtheorem{theorem}{Theorem}

\makeatletter
\newcommand\thefontsize{The current font size is: \f@size pt}

\def\BibTeX{{\rm B\kern-.05em{\sc i\kern-.025em b}\kern-.08em
    T\kern-.1667em\lower.7ex\hbox{E}\kern-.125emX}}
\begin{document}

\title{Accelerating Data Loading in Deep Neural Network Training}

\author{\IEEEauthorblockN{Chih-Chieh Yang}
\IEEEauthorblockA{\textit{Data Centric Systems} \\
\textit{IBM T. J. Watson Research Center}\\
Yorktown Heights, NY, USA \\
chih.chieh.yang@ibm.com}
\and
\IEEEauthorblockN{Guojing Cong}
\IEEEauthorblockA{\textit{Data Centric Systems} \\
\textit{IBM T. J. Watson Research Center}\\
Yorktown Heights, NY, USA \\
gcong@us.ibm.com}
}

\maketitle

\begin{abstract}
Data loading can dominate deep neural network training time on large-scale systems. We present a comprehensive study on accelerating data loading performance in large-scale distributed training. We first identify performance and scalability issues in current data loading implementations. We then propose optimizations that utilize CPU resources to the data loader design. We use an analytical model to characterize the impact of data loading on the overall training time and establish the performance trend as we scale up distributed training. Our model suggests that I/O rate limits the scalability of distributed training, which inspires us to design a locality-aware data loading method.  By utilizing software caches, our method can drastically reduce the data loading communication volume in comparison with the original data loading implementation. Finally, we evaluate the proposed optimizations with various experiments. We achieved more than 30x speedup in data loading using 256 nodes with 1,024 learners.
\end{abstract}


\section{Introduction}
% What is deep learning?

% DNN training current status
Deep Neural Network (DNN) models work incredibly well in real world scenarios, such as image classification, speech recognition, and autonomous driving. However, DNN training can  take a long time --- days, weeks, or even months --- which makes it difficult to optimize and re-train models. Researchers have devoted much effort into speeding up DNN training. On the hardware side, vendors incorporated stronger machine learning capabilities to processor architecture design, and introduced special-purpose accelerators for machine learning~\cite{jouppi2017datacenter}. On the software side, researchers developed optimized libraries such as CUDNN~\cite{chetlur2014cudnn} and MKL-DNN~\cite{mkldnn2019}; created easy-to-use frameworks such as Caffe~\cite{jia2014caffe}, PyTorch~\cite{paszke2017automatic} and Tensorflow~\cite{tensorflow2015whitepaper}; and invented new learning algorithms that converge faster. They also ran DNN training on large-scale high-performance computing (HPC) systems, which leads to many interesting research problems. For example, scaling up while maintaining convergence rate, finding a good computation-to-communication ratio, and synchronizing results efficiently. Finding solutions to these problems has reduced training time immensely --- take Imagenet-1K~\cite{ILSVRC15} training with ResNet50~\cite{he2016deep} model for example, the state of the art training time is reduced from an hour~\cite{goyal2017accurate} to minutes~\cite{You_2018,jia2018highly,ying2018image,mikami2018imagenetresnet50} within a year.

% The significance of data-loading
In large-scale distributed DNN training, we can break down the training time into three major components: computation time, communication time and data-loading time. 
While the former two draw great attentions from researchers, data-loading time is often omitted in the literature, since different techniques exist that circumvent the data-loading problem. For example, it is possible to cache a dataset entirely in fast local storage such as SSD or DRAM, instead of loading data from a network-based file system that has higher I/O overhead. One can also preprocess a dataset to reduce its size so it fits in a local storage, if the original dataset is too large. However, such techniques do not always apply. Fast local storage may not fit the whole dataset even after preprocessing, and preprocessing may take a long time and cause loss of pertinent information.

Considering common usage scenarios in HPC, it is important to design efficient methods to load data from a network-based file system or a data server, so that the data loading time does not become a bottleneck in DNN training. In this work, we propose data loader optimizations and bandwidth requirement optimizations to significantly improve data loading time in large-scale distributed DNN training. We evaluated the proposed methods on PyTorch, and found our methods can provide more than 30x speedup in data loading using 256 nodes with 1,024 learners. For the Imagenet-1K classification, our optimizations give 92\% improvement in per epoch training cost over the regular distributed training implementation. 

Our contributions are as follows:
\begin{itemize}
\item Data loader optimizations that improve data loading cost;
\item A locality-aware data loading method which greatly reduces bandwidth requirement and improves the scalability;
\item An analytical model that models the cost and establishes the performance trend in different system scales;
\item A performance evaluation that showcases the effects of our proposed optimizations.
\end{itemize}

The rest of the paper is organized as follows. 
Section~\ref{sec:background} describes necessary background information.
In Section~\ref{sec:data-loader-optimizations}, we present data loading optimizations that better utilize CPU to reduce overhead. Next, we explain the performance trend when scaling up with an analytical model in Section~\ref{sec:performance-model}. 
In Section~\ref{sec:locality-aware-data-loading}, we propose a locality-aware data loading method which greatly reduces the bandwidth requirement of mini-batch SGD. In Section~\ref{sec:experiments}, we show improvements brought by our optimizations. In Section~\ref{sec:related-work}, we summarize related works. Finally, we draw the conclusions and describe the future work in Section~\ref{sec:conclusion}.

\section{Background}\label{sec:background}

\subsection{Mini-batch SGD}\label{sec:sgd}

% What is mini-batch SGD
Mini-batch stochastic gradient descent (SGD)~\cite{dekel2012optimal} approximates the optimal solution of an objective function. It iteratively feeds a mini-batch (i.e. a set of random samples of a dataset) to a neural network model for forward propagation, and computes the loss from the output and the target values. Next, it performs backpropagation to compute gradients, which are then used to update the model weights. The frequency of model updates, depending on the choice of mini-batch sizes, can affect the convergence rate. 

In this context, we use a \emph{step} to refer to training a single mini-batch, and an \emph{epoch} to refer to training the whole dataset in multiple steps. Depending on the dataset and the model, it takes various number of epochs for the training to converge.

% Distributed training
Mini-batch SGD can be parallelized in a data-parallel fashion. A typical implementation uses multiple distributed processes (hereafter referred to as \emph{learners}), each with a copy of the model. The learners perform a step of mini-batch SGD collectively with the following procedure: 

\begin{enumerate}
\item Each learner acquires the same global mini-batch sequence (a sequence of sample indices instead of the actual samples) that all learners will collectively load.
\item Each learner takes an even-sized disjoint slice of the global mini-batch sequence. 
\item Each learner loads samples of its slice from the data source (e.g. a network file system) to form a local batch.
\item Each learner trains with its local batch independently to compute local gradients.
\item All learners synchronize (i.e. all-reduce) to produce the global gradients of the current step. 
\item Each learner updates the model weights with the same global gradients.
\end{enumerate}

% Problem of scaling mini-batch SGD
The above mentioned procedure is \emph{synchronous} mini-batch SGD, as the model is globally synchronized in each step. This data-parallel approach scales well when the mini-batch size is large enough to have a good compute-to-communication ratio. 
However, two issues exist: (1) In general, using a larger mini-batch makes it harder to generalize; (2) Synchronization overhead increases with the number of learners. 
There has been rigorous research work on the two issues, and the recent results in~\cite{Kurth_2017,kurth2018exascale} show that the distributed SGD can scale to extreme scale systems.

\subsection{Cost and scalability of distributed training}\label{sec:cost-of-distributed-training}

Considering the cost of a single step in synchronized mini-batch SGD training, the are three major components: computation time (forward propagation and backpropagation), communication time (synchronization of gradients), and data loading time.

% Training cost
Forward propagation and backpropagation to produce the local gradients take up most of the computation time in training. The actual cost depends on the complexity of the model used and the local batch size. In a small scale, the computation time usually dominates the overall cost in deep learning. Various techniques can be applied to reduce this cost, such as using optimized DNN library like CUDNN~\cite{chetlur2014cudnn} for Nvidia GPUs, and using lower precision floating arithmetic. In the foreseeable future, new hardware architectural advances will improve the computation time faster than the other components. 

% Communication cost
In synchronous mini-batch SGD, the synchronization of gradients happens per step, since the model weights must be updated before the next step starts. There have been many recent research works on improving synchronization algorithms including~\cite{mikami2018imagenetresnet50, jia2018highly}. The gradients of different model layers can also be synchronized separately~\cite{alex2018horovod}. The communication and computation cost are tightly coupled, as there are optimization techniques such as layer-by-layer synchronizations. In the rest of this paper, we use \emph{training time} to refer to the overall cost of computations and communication, and discuss its relationship with \emph{data loading time}.

% Data Loading Cost
Data loading in the machine learning context refers to the actions required to move data samples from a storage location to form a batch in the memory co-located with the compute units for training. The I/O cost (typically read-only) of moving data samples depends on the bandwidth of the storage system. Other than the I/O cost, to make data usable in training, there is often some preprocessing or data augmentation needed, depending on the training requirements. Take an image classification task for example, one needs to decompress the image files, to randomly clip and resize the image, and perform other image transformations. These operations can be time-consuming. 

To understand the overall cost in deep learning applications, let us first consider a single learner case. A learner waits if data is not prepared in time, as the training progress depends on input data.  If a learner performs data loading and computation sequentially, there will be gaps between computation tasks caused by data loading. In comparison, a common practice is to, in a background process (or a thread), use \emph{prefetching} to overlap data loading with training. The data loading overhead can be partially or completely hidden. 

Now, let us consider scaling up distributed training of a fixed-sized input dataset. Suppose our training is data parallel, when we get more computing resources, the training time decreases, as the computation can easily be parallelized. In contrast, the data loading cost also decreases initially, because more processes participate in preprocessing the same amount of data, and more nodes can load data in simultaneously, which increases the effective bandwidth. However, the bandwidth of the storage system is eventually upper-bounded.

% Data: 20190331-load-experiments-with-full-transform
\begin{figure}[t]
\centering
\includegraphics[width=.45\textwidth]{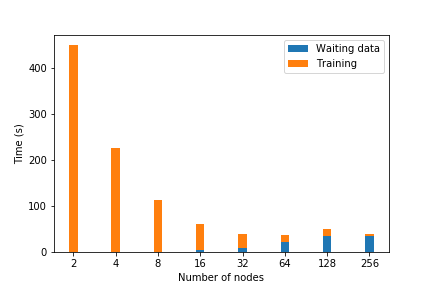}
\caption{Average epoch time to train ResNet50 with Imagenet-1K dataset in different scales on LLNL Lassen. The cost stopped decreasing when the data loading overhead stopped scaling. }
\label{fig:bandwidth-limitation}
\end{figure}

Figure~\ref{fig:bandwidth-limitation} shows the data loading scalability problem. On LLNL's Lassen system, we used distributed processes to train ResNet50 with Imagenet-1K dataset. There were four learners per node. The batches were globally and randomly shuffled, and the local batch size of a learner was fixed at 128. The global batch size increases with the number of nodes.

We measured the average \emph{training time} per epoch as shown in the orange bars; and the average \emph{waiting time} for data to be ready for training as shown in the blue bars. The sum of the two is the average cost of an epoch. Since data loading is overlapped with training, the time to wait for data would appear only when data loading overhead was not completely hidden.  
We can see that for 2, 4, and 8 nodes, the waiting for data was minimal and the performance scaled well. However, while the training time kept decreasing with more participating nodes, the data loading cost could no longer be fully hidden. The waiting time was significant starting from the 16-node case, and eventually dominated the cost as we added more nodes.
This is because even though the load volume per node decreased while scaling up, the overall data supplying rate could not catch up with the consuming rate, as a result, the data loading time stopped decreasing, and the cost stopped scaling down.

These observations motivated us to attack the data loading problem from two different angles: (1) reducing the data loading cost to result in an overall improvement for all cases, and (2) reducing the data loading volume, so that the storage system's limited capability becomes less of a problem. We address the former aspect in Section~\ref{sec:data-loader-optimizations} and the latter in Section~\ref{sec:locality-aware-data-loading}.

\section{Data loader optimizations}\label{sec:data-loader-optimizations}

To optimize data loading cost, we need to identify the overhead of data loading in a finer granularity. An illustrative typical learner execution timeline, similar to visualization of profiling tools such as nvprof, is shown in Figure~\ref{fig:execution-timeline}. Each timeline represents a computing resource. The colored bars in a timeline denote tasks being performed and the white space represents idling. 

The main process (the middle timeline) drives the training progress and interacts with the data loader worker to prescribe batch-loading requests and retrieve data; it also interacts with GPU to perform computations. As illustrated, if we consider only the training time meaningful work, there can be overhead due to data loading. In the following subsections, we discuss different optimization strategies that address different overheads.

\noindent
\begin{figure*}[th]
\centering
\includegraphics[width=.8\textwidth]{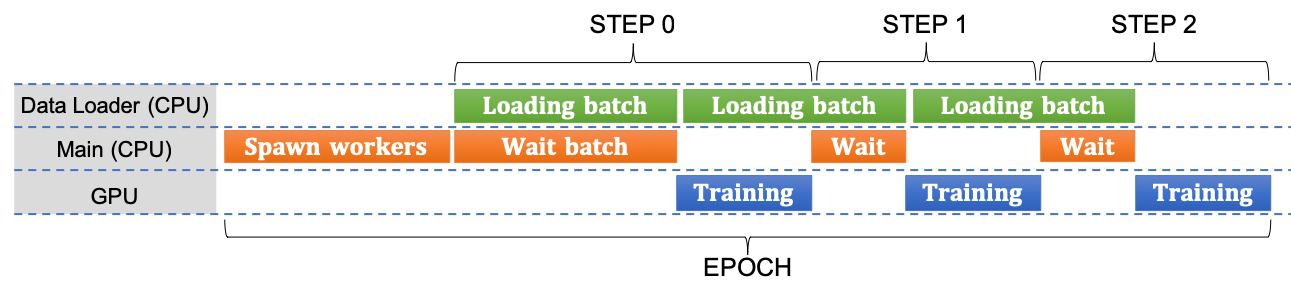}
\caption{Illustrative execution timeline of a learner.}
\label{fig:execution-timeline}
\end{figure*}

\subsection{Multiprocessing}\label{sec:optimizations-multiprocessing}

The time to load a batch can be significant, as illustrated in Figure~\ref{fig:execution-timeline} as the green bars. While we can use a background worker to prefetch batches to hide some overhead, adding more workers can overlap the loading of different batches and improve performance further.

PyTorch data loader implementation can spawn concurrent background worker processes to load multiple batches in parallel and maximize data loading throughput. The main process communicate with workers through \texttt{multiprocessing.Queue} instances. The main process prefetches data by submitting more batch-loading requests than its immediate demand.
When using more workers, the throughput increases because workers simultaneously load samples from the data source, increasing the effective bandwidth. It also parallelizes the preprocessing of batches. 

\begin{figure}[t]
	\centering
	\includegraphics[width=.40\textwidth]{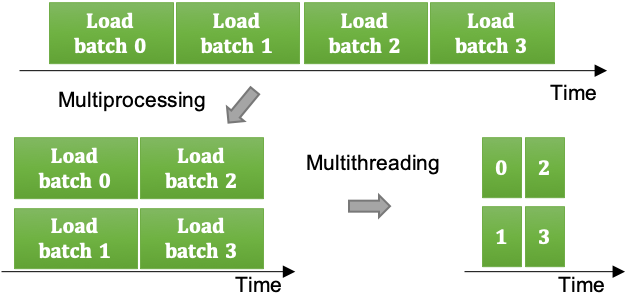}
	\caption{Parallelizing data loading. Multiprocessing allows overlapping of batch loading while multithreading further reduces batch loading time by utilizing parallelism within a batch.}
	\label{fig:parallelization-optimizations}
\end{figure}

\subsection{Multithreading}\label{sec:optimizations-multithreading}
While multiprocessing loads multiple batches in parallel, there exists untapped parallelism within the loading of a batch. It is often the case that the preprocessing of individual samples can all be done independently in parallel. Thus, we can use multithreading to parallelize sample preprocessing within a batch, so that the loading time can reduce further. Figure~\ref{fig:parallelization-optimizations} illustrates the effects of multiprocessing and multithreading. Multiprocessing overlaps batch loading across processes, while multithreading within a worker shortens loading time per batch by preprocessing samples in parallel.
 
While multithreaded data loading exists in other deep learning frameworks, in Pytorch, we have to modify the data loader implementation to create a \texttt{ThreadPoolExecutor} instance along with a data loader worker. Instead of loading samples and preprocessing them one-after-one sequentially in a single thread, we use \texttt{ThreadPoolExecuter.map()} to load samples in a batch concurrently. Note that, due to the python global interpreter lock (GIL) issue\cite{pywikiglobalinterpreterlock}, multithreading works well only if the preprocessing pipeline stages call native library routines and release GIL correctly. In experiments of Imagenet-1K training, system calls such as file I/O and the image transformations do release GIL, and we can see significant performance improvement with multithreading (see Section~\ref{sec:effects-of-optimizations}).

\subsection{Caching}\label{sec:optimizations-caching}

The data access pattern of mini-batch SGD is repetitive and random. The learners collectively load the same training dataset with a randomized sequence every epoch. Since the samples are reused, there is temporal locality that we can utilize to improve performance.

% Cache in local DRAM/SSD
We can allocate a software cache either in memory or in a high speed local storage such as an SSD of a compute node to store samples that have been loaded in earlier epochs. The cached samples will be used again in subsequent epochs. With a cache hit, a learner loads a sample with a much shorter latency. It also reduces the number of accesses to the storage system, making the I/O bandwidth less likely to be saturated.

While caching in memory grants optimal data access time,  training datasets that are too large to fit in the local DRAM can be cached in SSDs. For very large datasets that do not fit in the local cache, caching a partial subset locally can still improve performance although the improvement can be limited.  For example, considering a compute node that caches 10\% of the training dataset, the cache hit rate is 0.1. In other words, 90\% of the samples are still loaded from the storage system. 

% Distributed caching (DeepIO and Kurth's paper)
To avoid being limited by the local cache size, all the participating compute nodes can share their local caches with each other to form an \emph{aggregated cache} that is many times larger than individual caches, similar to the high-speed parallel data staging method mentioned in~\cite{kurth2018exascale}. With the aggregated cache, compute nodes may cache disjoint partitions of a large dataset. We refer to this technique as \emph{distributed caching}.

Distributed caching changes data-loading in mini-batch SGD. 
During training, a sample load can be a local cache hit, a remote cache hit, or a cache miss  served by the storage system. The local cache hit rate is likely to be small, assuming the local cache only holds a small subset of the whole training dataset. But the remote cache hit rate can be very high, if the aggregate cache holds most of the dataset. The cache miss rate can be zero, if the whole training dataset is collectively cached. The technique bears a lot of similarities to the hierarchical hardware caches in modern processor architectures. 

With distributed caching, learners can exchange cached samples to create their local batches (i.e. a slice of a mini-batch sequence). The exchanges utilize the high-speed network among compute nodes instead of loading from the storage system. In this way, the storage system bandwidth is no longer a bottleneck. However, the bandwidth among compute nodes, albeit typically larger, can still be a limiting factor, because the data loaded collectively per epoch is still close to the whole dataset size. We model the cost of distributed training in the next section, and in Section~\ref{sec:locality-aware-data-loading}, we propose a new data loading method that reduces the bandwidth requirement of distributed mini-batch SGD.

\section{Performance Model}\label{sec:performance-model}

We present a simple analytical model to help analyze the cost in different system scales. The model contains the following parameters (we denote uppercase letters to represent constant values and lowercase letters to represent variables):

\begin{itemize}%[label={}]
  \item[$D$:] the dataset size. To simplify the analysis, we assume that the data samples are the same size, and $D$ equals the total number of samples in the dataset.
  \item[$p$:] the number of participating compute nodes.
  \item[$V$:] the maximum training rate of a compute node. 
  \item[$R$:] the I/O rate from the storage system. We let this be the maximum loading rate.
  \item[$R_{c}$:] the I/O rate from the remote caches. We can reasonably assume $R_{c}$ is much larger than $R$ due to the high-speed interconnection that HPC systems typically have. 
  \item[$U$:] the maximum preprocessing rate of a compute node. 
  \item[$\alpha$:] the ratio of the cached subset (in the aggregated cache) to the whole dataset. 
\end{itemize}

From the discussion in Section~\ref{sec:cost-of-distributed-training}, we know that when data loading is overlapped with training, the true overall cost is the larger of the training cost and the data loading cost. The training cost and the data loading cost of an epoch can then be derived from: 
\begin{align}
\textsf{Training time} &= \frac{D}{p \cdot V} \label{eq:training-time}\\
\textsf{Sample I/O time} &= \frac{D}{R} \label{eq:sample-load-time}\\ 
\textsf{Sample preprocessing time} &= \frac{D}{p \cdot U} \label{eq:preprocessing-time}\\
\textsf{Data loading time} &= (\ref{eq:sample-load-time}) + (\ref{eq:preprocessing-time}) \notag\\
                          &= \frac{D}{R} + \frac{D}{p \cdot U} \label{eq:data-loading-time}
\end{align}

Now, let us revisit Figure~\ref{fig:bandwidth-limitation}. There was only data loading without training in the experiment. In the plot, the data loading cost was high initially when there were few nodes, but it decreased when more nodes were added until it hit a plateau. In (\ref{eq:data-loading-time}), the preprocessing time decreases when $p$ increases. It eventually becomes insignificant, relative to the constant sample I/O time. And the data loading costs at least $\frac{D}{R}$ which is a constant. This explains the plateau.

We can also determine the true cost by comparing the training cost and the data loading cost. To simplify the analysis, we assume that the preprocessing rate is much higher than training rate (i.e. $U \gg V$), and $p$ is also large enough so that the preprocessing cost is relatively insignificant. Thus, we only concern the relationship between (\ref{eq:training-time}) and (\ref{eq:sample-load-time}). If the training time dominates the true cost:
\begin{align}
(\ref{eq:training-time}) & \geq (\ref{eq:sample-load-time}) \notag \\
\frac{D}{p \cdot V} & \geq \frac{D}{R} \notag \\
p & \leq \frac{R}{V} \label{eq:condition}
\end{align}

From (\ref{eq:condition}), we know that for small $p$, the training time dominates. As $p$ increases, more computing resources can be used to reduce the training cost while the sample I/O time remains constant. The \emph{true cost} per epoch can thus be expressed as the following:

\begin{equation}
\textsf{True cost} = \begin{cases}\begin{aligned}
 \frac{D}{p \cdot V} & \quad \textsf{for} & p & \leq \frac{R}{V}  \\
 \frac{D}{R}          & \quad \textsf{for} & p & > \frac{R}{V} 
\end{aligned}
\end{cases}
\end{equation}

Now, let us consider distributed caching. We assume that the cost of local cache hits is insignificant. The optimization does not change the training cost or the preprocessing cost but only affects the sample I/O cost:
\begin{align}
\textsf{Sample I/O time} &= 
\underbrace{\frac{(1-\alpha) \cdot D}{R}}_{\textsf{Storage system}} + 
\underbrace{\frac{\alpha \cdot D }{R_{c}}}_{\textsf{Remote caches}} \cdot \underbrace{\frac{p - 1}{p}}_{\textsf{Local cache miss}}  \label{eq:sample-load-time-caching} 
\end{align}

We can know two things from (\ref{eq:sample-load-time-caching}): (a) The local cache miss rate $\frac{p-1}{p}$ can be very high when $p$ is large so local cache hits do not help very much although they are fast; (b) Both $\alpha$ and $R_{c}$ have to be large for distributed caching to perform very well. While scaling up, it is easy to have a large $\alpha$ since scaling up increases the amount of aggregated memory to store a fixed-sized dataset. $R_{c} \gg R$ is also a reasonable assumption in modern HPC systems. However, $R_{c}$ does not grow linearly with $p$, and eventually the performance scaling is limited by the bandwidth among compute nodes.

\section{Locality-aware Data Loading}\label{sec:locality-aware-data-loading}

The optimizations described in Section~\ref{sec:data-loader-optimizations} improve the data loading rate of learners. However, as illustrated in both in the experiment (Figure~\ref{fig:bandwidth-limitation}) and through performance modeling, a distributed learning application scales only as far as the storage system's capability allows. We need a way to reduce the bandwidth requirement of distributed DNN training to overcome the limitation.

In the rest of this section, we describe a data loading method which adds a locality-aware flavor to distributed caching. It not only reduces data loading from the storage system, but also minimizes overall data loading volume to a fraction of the dataset size. 

Instead of exchanging among learners to form designated mini-batch slices, learners can assemble a mini-batch from their locally cached samples to greatly reduce data loading. A key property of SGD makes it possible: for a given global mini-batch sequence, as long as all the samples in such sequence are used in the training step, the ordering of the samples within the global batch does not affect the training results after synchronization (i.e. all-reduce operation).

\subsection{Methodology}\label{sec:methodology}

%\subsubsection{Populating the cache}
As in distributed caching, learners must populate their local caches before the locality-aware data loading method can be applied. This can either be a cache populating phase before training, or caching the samples loaded from the storage system on-the-fly during the first epoch. As long as the cached subsets are disjoint, how samples are cached is not important, since the mini-batch sequences are randomly shuffled. However, it may be advantageous to populate the caches in a way that sample locations (i.e. the nodes samples are cached) can be easily determined to avoid extra book-keeping. We assume a \emph{cache directory} exists for tracking sample locations, and the directory is duplicated across all learners and stays the same (i.e. no cache replacement) after populating caches in the first epoch. 

% Locality aware data loading
With locality-aware data loading method, in a training step, a learner goes through a given predefined global mini-batch sequence and look for samples that are cached locally and trains with them. Given the number of compute nodes $p$, if the whole dataset is evenly split among all caches, and a global mini-batch sequence is uniform-randomly sampled, a compute node should find close to $\frac{1}{p}$ of the global mini-batch in its local cache. It can then use these samples in the training step as its local batch. The results of training with this subset of the global mini-batch sequence are the learner's contribution in the training step. In Section~\ref{sec:proof}, we prove that training with this method produces equivalent results to the regular method.

% Imbalanced workload distribution
When learners look for samples locally cached, they may find themselves caching varying sized subsets of the global mini-batch sequence. In other words, the sample distribution can be imbalanced. Letting learners train with imbalanced local batches, while giving the same training results and potentially performing zero remote data loading, can cause some learners to become stragglers and increase the training time of a step in synchronous SGD. We need load balancing for optimal performance, and we discuss this further in Section~\ref{sec:load-imbalance}.

Assuming the caches have been populated with samples, the procedure of locality-aware data loading is as follows:

\begin{enumerate}
\item Get a global mini-batch sequence that is the same across all learners.
\item Determine the sample distribution of the global mini-batch among the distributed learners.
\item Determine data loading for either samples missing from the aggregated cache or load balancing.
\end{enumerate}

First, all learners get the same global mini-batch sequence. 
Next, each learner independently goes through the global sequence and determines the sample distributions by looking up the cache directory. 
Then, the learners need to agree on how to load samples locally so that they collectively assemble the global mini-batch. Samples not in the caches are loaded from the storage system. As for load-balancing, the learners can exchange data to achieve load balance, or they can load from the storage system. If learners exchange samples for load balancing, it creates point-to-point communication traffic.  We provide both theory and simulation results in Section~\ref{sec:load-imbalance} to show that this traffic is a small fraction of the data movement that the regular loading method requires.

\begin{figure}[t]
\centering
    \begin{minipage}{.49\textwidth}
        \centering
        \includegraphics[width=.4\linewidth]{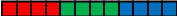}
        \caption{Conventional method: learners load even-sized slices.}
        \label{fig:minibatch-slices}
    \end{minipage}    

    \vspace{1em}    
    
    \begin{minipage}{.49\textwidth}
        \centering
        \includegraphics[width=.4\linewidth]{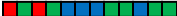}
        \caption{Locality-aware method: sample distribution in learner caches.}
        \label{fig:sample-distribution}
    \end{minipage}
\end{figure}

Figure~\ref{fig:minibatch-slices} and~\ref{fig:sample-distribution}  illustrate the differences between using the regular mini-batch SGD and the locality-aware method. We assume 3 learners --- Red, Green, and Blue --- collectively load a global mini-batch of 12 samples. In the regular method (Figure~\ref{fig:minibatch-slices}), the global mini-batch sequence is split into multiple slices, and each learner loads a slice of 4 samples to train with before synchronization. In the locality-aware method, the learners look for the \emph{locally cached} samples that belongs to the mini-batch. In Figure~\ref{fig:sample-distribution}, Red has 2 samples, Green has 6 samples, and and Blue has 4 samples in their local caches. A way to balance the load is to let Red load 2 samples from Green before training. The total volume loaded for this global mini-batch is 2$\div$12 $\approx$ 17\% of the regular method.

\subsection{Proof of equivalence}\label{sec:proof}

Here we give a formal proof that the locality-aware data loading method produces the same results as the regular loading method. Consider the following optimization problem solved by mini-batch SGD:
\begin{equation*}
	\min\limits_{w \in \mathcal{X}} F(w)
\end{equation*}
where $F:\RR^m \rightarrow \RR$ is continuously
differentiable but not necessarily convex over $\mathcal{X}$, and
$\mathcal{X}\subset \RR^m$ is a nonempty open subset. The objective
$F$ can be seen as the empirical risk \\ $F(w) = n^{-1}\sum_{i=1}^n
g_i(w, x_i)$. Here $x_i$,  $1\leq i \leq n$, are data samples.  

The regular data loader and locality-aware dataloader implement two
sampling schemes for $P$ learners. We call the regular one \Reg, and
the locality-aware one \Loc. 
\begin{theorem} 
\label{th:equiv} 
Assuming the same sequence of random numbers are generated
  for \Reg and \Loc, then distributed minibatch SGD produces the same
  $w$ with both sampling schemes after the same number of training steps.
\end{theorem}
\begin{proof}

 By induction. Denote $x_t^k$ as the $k$-th ($1\leq k \leq B$) sample in the $t$-th mini-batch, where $B$ is the size of the global mini-batch.  Since the mini-batch is evenly distributed to each learner
 in a block fashion, at the $j$-th ($1\leq j \leq P$) learner $L_j$, the local mini-batch includes
 $\{x_t^{(B/P)*(j-1)+1}, \cdots, x_t^{(B/P)*(j)} \}$. Assume after
 $t=s$, $s\geq 1$, $w$ is the same under \Reg and \Loc. 
 Then at step $s+1$, \Reg produces a global mini-batch sequence $\{x_{s+1}^1, x_{s+1}^2,
 \cdots, x_{s+1}^B\}$. 
 
With \Reg the global mini-batch sequences are block-distributed, and at learner $L_j$, the local batch is: 

\begin{align*} 
\{x_{s+1}^{(B/P)*(j-1)+1}, x_{s+1}^{(B/P)*(j-1)+2}, \cdots, x_{s+1}^{(B/P)*(j)} \}
\end{align*}

So the local gradient is:
 
\begin{align*} 
 \nabla F(w:\{x_{s+1}^{(B/P)*(j-1)+1}, x_{s+1}^{(B/P)*(j-1)+2}, \cdots, 
 x_{s+1}^{(B/P)*(j)}\})  \\ 
 =\sum_{jj} \nabla F(w:x_{s+1}^{(B/P)*(j-1)+jj})
\end{align*}
 
And the global gradient after reduction is: 

\begin{equation*} 
\nabla_{\Reg} =\sum_j \sum_{jj} \nabla
 F(w:x_{s+1}^{(B/P)*(j-1)+jj})
\end{equation*}

  Since \Loc uses the same random number sequence, it produces the 
 same global mini-batch sequence $\{x_{s+1}^1, x_{s+1}^2,
 \cdots, x_{s+1}^B\}$ as \Reg. However, due to locality optimization, the
 sequence is not distributed to the learners in a block fashion. In
 fact, the local batch may actually have different
 sizes. From the convergence perspective, locality-aware optimization in effect permutes the sampling sequence $\{x_{s+1}^1, x_{s+1}^2,
 \cdots, x_{s+1}^B\}$ into $\{x_{s+1}^{g_1}, x_{s+1}^{g_2},
 \cdots, x_{s+1}^{g_B}\}$, and distributes it unevenly in a block
 fashion to the learners. Suppose learner $L_j$ gets samples $g_{j_b}$ to $g_{j_e}$.  Then the local gradient is: 
\begin{align*}
 \nabla F(w:\{x_{s+1}^{g_{j_b}}, x_{s+1}^{g_{j_b}+1}, \cdots,
 x_{s+1}^{g_{j_e}}\})=\sum_{g_{j_b}\leq jj \leq g_{j_e}} \nabla
 F(w:x_{s+1}^{jj})
\end{align*} 

And the global gradient after reduction:
\begin{align*}
 \nabla_{\Loc} =\sum_j \sum_{g_{j_b}\leq jj \leq g_{j_e}} \nabla
 F(w:x_{s+1}^{jj})
\end{align*} 
 
By the commutative law of addition, $\nabla_{\Loc} = \nabla_{\Reg}$. Therefore, $w_{s+1}$ of the two methods are the same. Obviously, the base $w_1$ is the same for both sampling schemes. This
completes our proof. 
\end{proof}

Theorem~\ref{th:equiv} shows that our locality-aware data loading scheme produces the same gradients as the original approach for each step in distributed SGD.  In current practice, batch normalization is frequently used to improve training accuracy and time.  In theory, batch normalization should be applied to the whole mini-batch. In this case, Theorem~\ref{th:equiv} still holds.  If batch normalization is applied to each local part of the mini-batch, the mean and the variance are obviously different from the original data loading scheme. However, from the training perspective, the impact of our locality-aware scheme on batch normalization is similar to that of using a different random permutation sequence.  It should have minimal impact on training results. This is confirmed by our experimental results. 

\subsection{Load Imbalance}\label{sec:load-imbalance}

Here, we discuss the load imbalance of the locality-aware data loading method. We first analyze the data imbalance among the caches for the learners. The distribution of data samples in a global mini-batch to caches is a random process. To characterize the amount of data samples of a global mini-batch in the cache of a certain learner, we consider the process of uniformly-at-random placing $b$ balls in $p$ bins. Let $M$ be the random variable that counts the maximum number of balls in any bin. Then $Pr[M>K_\alpha] = o(1)$ for $\alpha>1$ and $K_\alpha = \frac{b}{p} + \alpha \sqrt{2\frac{b}{p} \log p}$, with $p\log p \ll b \leq  p \cdot \text{polylog}(p)$ (see Theorem 1 of~\cite{raab1998balls}).  The imbalance in the amount of data samples for a mini-batch in theory is unlikely to be large. 

% Data from 20190128-imbalance-sim
\begin{figure}[t]
\centering
\includegraphics[width=.40\textwidth]{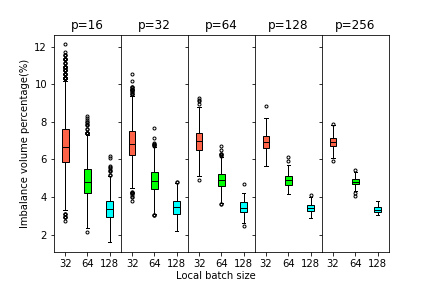}
\caption{Simulated imbalance of the global mini-batch sample distribution in distributed caching. p is the number of compute nodes.}
\label{fig:sample-distribution-imbalance}
\end{figure}

We ran simulations to show the traffic volume needed to balance the batch samples. Different local batch sizes and different number of compute nodes are used. The simulation started with a fixed sized dataset evenly partitioned and distributed to $p$ compute nodes. Then, mini-batch sequences were generated, and the sample distributions were determined. The imbalance traffic volume percentage is calculated by summing the deficits of every learner and then divided by the mini-batch size. We collect the imbalance numbers of many steps and render the box plot as shown in Figure~\ref{fig:sample-distribution-imbalance}. 

We can make two observations from the figure. First, the imbalance depends on the local batch size.  For example, the green boxes are the results of the same local batch size 64. And they have very close median values across different configurations. The same applies to the other local batch sizes. Second, the imbalance is in general a small percentage for moderate to large local batch sizes. The median values of the imbalance percentage for the local batch size 32, 64, and 128 are approximately 6.9\%, 4.8\%, and 3.4\%, respectively.

Both the theory and simulation results show that the load imbalance of the locality aware data-loading is small. Still, imbalance in the amount of data present in the cache of each learner creates imbalance in computation time for forward and backward propagation in training. To achieve perfect loading balancing, learners with data surplus need to send some data samples to learners with deficit. These data transfers incur communication among the learners, and we want to minimize the number of transfers (since the total amount of data measured in bytes being transferred in any scheme is the same).  This optimization problem is equivalent to an existing problem and turns out to be $NP$-complete (see~\cite{chen2008minimum}). 

We propose an approximation algorithm. Its formal description is given in Algorithm~\ref{alg:balance}. In the algorithm, we build two heaps, one for learners with surplus $\mathcal{H}_s$, and the other for learners with deficit $\mathcal{H}_d$.  Each heap element contains two items, \emph{imbalance} for the current imbalance in workload, and \emph{ID} for the learner.  The algorithm greedily finds the current largest imbalanced heap elements $h_s$ in $\mathcal{H}_s$ and $h_d$ in  $\mathcal{H}_d$, and records sending min($h_s$.imbalance, $h_d$.imbalance) amount of data samples from the learner with the surplus to the learner with the deficit in the schedule list $\mathcal{S}$.  The algorithm then updates the heaps, and continue. 

\begin{algorithm}[t]
\footnotesize
\begin{algorithmic}[1]
\caption{Balance ($p$, $\mathcal{L}$)}
\label{alg:balance}
        \STATE Make a surplus heap $\mathcal{H}_s$ of all surpluses in $\mathcal{L}$ in decreasing order
        \STATE Make a deficit heap $\mathcal{H}_d$ of all deficits in $\mathcal{L}$ in decreasing order
        \STATE $\mathcal {S} \leftarrow \{\}$
        \WHILE{$\mathcal{H}_s$ is not empty}
               \STATE $h_s \leftarrow$ find-max($\mathcal{H}_s$)
               \STATE $h_d \leftarrow$ find-max($\mathcal{H}_d$)
               \STATE $m  \leftarrow$ min($h_s$.imbalance, $h_d$.imblance)
               \STATE $h_s$.imbalance $\leftarrow h_s$.imbalance $- m$
               \STATE $h_d$.imbalance $\leftarrow h_d$.imbalance  $- m$
               \STATE $\mathcal {S}$.append($h_s$.ID, $h_d$.ID, $m$)
               \IF{$h_s$.imbalance $=0$}
               \STATE heap-remove($h_s$)
               \ELSE
               \STATE heap-decrease-key($h_s$)
               \ENDIF
               \IF{$h_d$.imbalance $=0$}
               \STATE heap-remove($h_d$)
               \ELSE
               \STATE heap-decrease-key($h_d$)
               \ENDIF
                \STATE $h_s \leftarrow$ heap-find-max($\mathcal{H}_s$)
               \STATE $h_d \leftarrow$ heap-find-max($\mathcal{H}_d$)
        \ENDWHILE %\label{line:sgd-ed}
        \RETURN {$\mathcal{S}$}
\end{algorithmic}
\end{algorithm}
 
Since with each heap-find-max operation on $\mathcal{H}_s$ the imbalance of at least one learner is removed, and the heap operation takes at most $\log p$ time, it is easy to see that Algorithm~\ref{alg:balance} runs in $O (p \log p)$ time. 

\begin{theorem}
 Algorithm~\ref{alg:balance} is a 2-approximation algorithm.
\end{theorem}
\begin{proof} 
In the worst case, the number of messages sent by Algorithm~\ref{alg:balance} is at most $p-1$ as each heap-find-max operation on $\mathcal{H}_s$ fixes one imbalanced learner, and the minimum of messages sent is $p/2$,  the approximation ratio is $\frac{p-1}{p/2} \approx 2$.
\end{proof}

We extend the performance model described in Section~\ref{sec:performance-model} to incorporate locality-aware data loading method. The training time and preprocessing time are the same as previously described. We focus on the sample I/O time here, since it dominates the cost when $p$ is large. Two new parameters are needed:

\begin{itemize}%[label={}]
  \item[$R_{b}$:] the I/O rate of data movements for load balancing. If we choose to load the samples from remote caches, we can let $R_{b} = R_{c}$.
  \item[$\beta$:] the load balancing traffic volume ratio to the a given dataset size. 
\end{itemize}

The sample I/O time using the locality-aware data loading method is:
\begin{align}
\textsf{Sample I/O time} &= 
\underbrace{\frac{(1-\alpha) \cdot D}{R}}_{\textsf{Storage system}} + 
\underbrace{\frac{\alpha \cdot D }{R_{b}} \cdot \beta}_{\textsf{Load balancing cost}}  \label{eq:sample-load-time-locality-aware} 
\end{align}

From the previous analyses, we know that $\beta$ is a small number (i.e. $ 0 \leq\beta \ll 1$) because load imbalance is unlikely to be large. We can see that (\ref{eq:sample-load-time-locality-aware})  differs from (\ref{eq:sample-load-time-caching}) only in the second term. When $p$ is large, $\frac{p-1}{p} \approx 1 \gg \beta $, thus, compared with distributed caching, the locality-aware data-loading method greatly reduces the I/O cost.

\section{Experiments}\label{sec:experiments}

\begin{figure*}[t]
	\begin{minipage}{.40\textwidth}
		\centering
		\includegraphics[width=\textwidth]{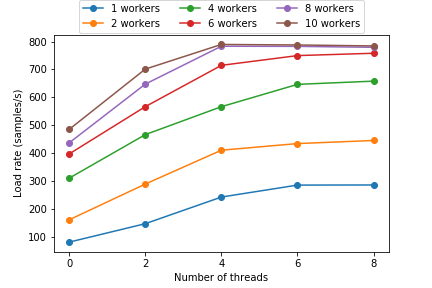}
		\caption{The Imagenet-1K sample loading rate of a single learner using different workers/threads combinations.}
		\label{fig:workers-threads-measurement}
	\end{minipage}
	\hspace{6em}
	\begin{minipage}{.40\textwidth}
	    \centering
		\includegraphics[width=\textwidth]{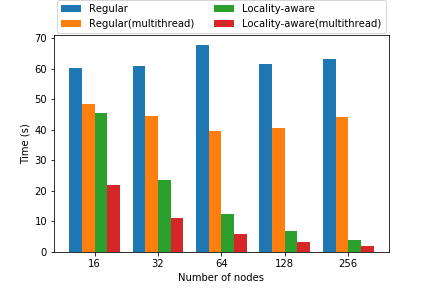}
		\caption{Cost to collectively load the Imagenet-1K dataset in different scales.}
		\label{fig:reg-vs-locality}
	\end{minipage}
\end{figure*} 	

We conducted  experiments on Lawrence Livermore National Lab's Lassen system using up to 256 nodes (1,024 GPUs). A compute node has two IBM POWER9 processors (44 cores in total), 256 GB system memory, 4 Nvidia V100 (Volta) GPUs, 16 GB memory per GPU, and Inifiniband EDR interconnect among compute nodes. Compute nodes have access to IBM Spectrum Scale (GPFS), a high-performance parallel file system. 

We studied a PyTorch implementation of Imagenet-1K classification using Resnet50 adopted from PyTorch examples~\cite{pytorchexamples2019}. The dataset, Imagenet-1K, contains around 1.28 million JPEG images, each is several hundred KBs. The total dataset size is about 150 GB. The distributed implementation spawns multiple learner processes, each associated with a GPU. The learners execute in a data-parallel fashion and synchronize with each other using NCCL library, which provides an optimized all-reduce operation for the synchronizations in training. 

We also include the data loading performance results of the UCF101 dataset~\cite{UCF101}. The dataset was originally videos and was converted into two image datasets: RGB and optical flow (referred to as FLOW) with approximately 2.5 million and 5 million images of average sizes 24.2 KB and 4.6 KB respectively.  We conducted data loading only (including I/O and video transformations) experiments with the optimized data loader to see how it performs for datasets other than Imagenet-1K.
 
To understand how our approach performs in loading very-large datasets, we used another 892 GB dataset generated from molecular dynamics (MD) simulations conducted using Multiscale Machine-Learned Modeling Infrastructure (MuMMI)~\cite{dinatale2019}. The dataset contains \textasciitilde 7M files that are derived MD trajectory frames. Each file contains a single frame of a constant size, 131 KB. The frames are stored in numpy array format and can be used in ML training directly after data loading. In other words, no sample pre-processing is required.

\subsection{Effects of Optimizations}\label{sec:effects-of-optimizations}

We examined the Imagenet-1K sample loading rate running a single data loading only learner (i.e. no training) with different numbers of workers and threads per worker to find a good combination. The case with zero thread (i.e. multithreading off) is the default PyTorch data loader. As shown in Figure~\ref{fig:workers-threads-measurement}, in general, the loading rate increases both with more threads and more workers. Our multithreading optimization granted better performance with relatively fewer workers, which is preferable because the overhead of spawning more workers increases quickly. The maximum loading rate measured is around 800 samples per second. 

Next, we compared the results of the regular PyTorch data loader with those of our locality-aware data loader. In each compute node, we created 4 learners to associate with 4 GPUs individually, and we let each learner spawn 10 background workers for it grants maximum sample loading rate in the previous experiment. The cache size of each learner is upper bounded at 25 GB but in most cases they use less than that for that we populate at most $\frac{1}{p}$ of the dataset per learner in the first epoch without cache replacement afterwards. The regular data loader read samples of a designated slice from a randomly permuted global mini-batch sequence in every step, while the locality-aware data loader went through a global mini-batch sequence, determined its contribution, and trained with mostly locally cached samples. We let the locality-aware data loader train with balanced local batches (using Algorithm~\ref{alg:balance}) to avoid negative  effects of stragglers. 

We ran a set of experiments in different scales. We removed the per-step synchronizations and kept only one in the end of each epoch. We experimented both with multithreading (4 threads each worker) and without multithreading to see if parallelized preprocessing helps the overall performance. We report the average time spent per-epoch excluding the first epoch, in which the caches were populated in the locality-aware method.

\begin{figure*}[t]
	\begin{minipage}{.40\textwidth}
		\centering
		\includegraphics[width=\textwidth]{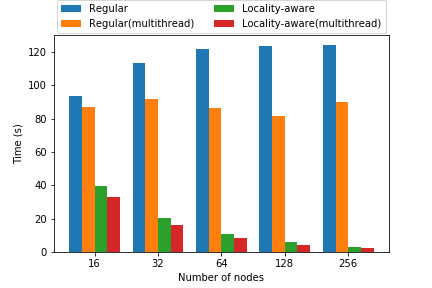}
		\caption{Cost to collectively load the UCF101-RGB dataset in different scales.}
		\label{fig:rgb-reg-vs-locality}
	\end{minipage}
	\hspace{6em}
	\begin{minipage}{.40\textwidth}
		\centering
		\includegraphics[width=\textwidth]{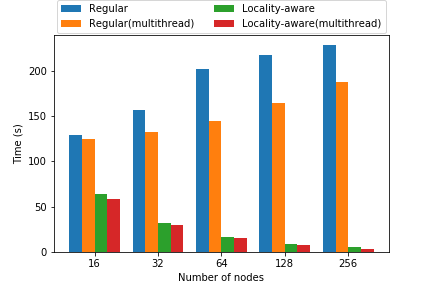}
		\caption{Cost to collectively load the UCF101-FLOW dataset in different scales. }
		\label{fig:flow-reg-vs-locality}
	\end{minipage}
\end{figure*}

\begin{figure}[t]
%\centering
\includegraphics[width=.40\textwidth]{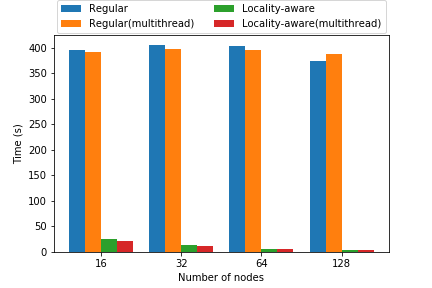}
\caption{Cost to collectively load the MuMMI dataset in different scales.}
\label{fig:mummi-reg-vs-locality}
\end{figure}

% Multithreading benefits
The results of loading Imagenet-1K in Figure~\ref{fig:reg-vs-locality} show that for the regular data loader, the cost did not decrease while scaling up. The slightly higher cost at 64-node was very likely caused by interference from other jobs. Regardless, since locality-aware data loaders fetched drastically fewer samples, data loading scales well with the number of learners.

In most configurations, for the same data loader and the same scale, a multithreaded trial finished sooner than a single-threaded one. We see a general improvement due to multithreading in all different numbers of compute nodes. For the regular data loader, multithreaded runs are 24\%--71\% faster; For the locality-aware data loader, multithreaded runs are 105\%--113\% faster. 

However, multithreading alone did not help in scaling up to more nodes when using the regular data loader. In contrast, the locality-aware data loading method clearly improved the scalability. Once the effective I/O bandwidth stopped scaling for the regular data loader --- as can be seen from the fact that the time did not reduce further even though more nodes participated --- the locality-aware data loader outperformed significantly, due to its lowered bandwidth requirement from reusing cached samples. At 256 nodes (1,024 learners), the locality-aware data loader achieved close to 34x speedup over the regular data loader at the same node count.

Figure~\ref{fig:rgb-reg-vs-locality} and~\ref{fig:flow-reg-vs-locality} show the results of loading the two sets from UCF101. The performance trend looks slightly different from Imagenet-1K results for the regular data loader. Without multithreading, the regular data loader spent more time per epoch to load both UCF101-RGB and UCF101-FLOW, albeit the volume to load per learner decreased as more compute nodes participated in loading. This trend also appeared in in loading UCF101-FLOW with multithreading. 
Since our jobs did not have exclusive access to the cluster during the experiment, we attribute the performance degradation to the interference of other jobs executing simultaneously that loaded from the GPFS. This phenomenon showcases again that at a very large scale, the conventional way of data loading does not scale.

In contrast to the regular data loader, our locality-aware data loader granted much better performance results for the data loading tasks of UCF101 in different scales. For UCF101-RGB, our optimized data loader is 2.8x--55.5x faster and for UCF101-FLOW, it is 2.2x--60.6x faster. 

For the largest dataset MuMMI, we can observe even more encouraging results using the locality-aware loading method as shown in Figure~\ref{fig:mummi-reg-vs-locality}. Our optimized data loader provides 18x, 35x, 70x, and 120x speedup over the regular data loader at 16, 32, 64 and 128 nodes correspondingly. We can see that the multithreading optimization does not affect the performance in a significant way, since the samples are numpy arrays and no pre-processing is needed after loading to the DRAM. 

\subsection{Imagenet-1K ResNet50 Training}\label{sec:imagenet-training}

We ran the Imagenet-1K classification using ResNet50 model to measure the performance results and the validation accuracy after 90 epochs in three different scales to compare the two data loader implementations. We enabled multithreading (4 threads per worker) in all runs. As we scaled up the distributed training, we also increased the global mini-batch size. We tried to reproduce the validation accuracy in~\cite{goyal2017accurate} for 8K mini-batch using the same fine-tuning techniques including batch normalization. For larger batch sizes, some of the known highest accuracy numbers achieved involve LARS~\cite{You_2018} and elaborate learning rate tuning, since our goal is not to achieve the highest accuracy, but to show comparable accuracy results, we did not implement those. In Table~\ref{tab:imagenet-accuracy},  we present the results. Using the locality-aware data loader resulted in comparable validation accuracy with that of the regular PyTorch data loader, as the differences are  below 1\%.

\begin{table}[t]
\centering
\caption{Imagenet-1K ResNet50 validation accuracy comparison between the regular data loader and the locality-aware data loader.}
\label{tab:imagenet-accuracy}
\begin{tabular}{c|r|c|c}
\hline
 \thead{Number of \\ nodes} &  \thead{Mini-batch \\ size} & \thead{Regular loader\\(\%)} & \thead{Locality-aware\\loader(\%)} \\
 \hline
 16 &  8,192     & 76.67 & 76.81 \\
 32 & 16,384     & 75.33 & 75.12 \\
 64 & 32,768     & 68.69 & 69.54 \\
 \hline
\end{tabular}
\end{table}

Figure~\ref{fig:imagenet-avg-epoch-time} shows the average time per-epoch of the runs. With training on GPUs, the data loading overhead should be hidden except for when $p$ is large. For 16 nodes, the GPU training time dominated the cost, and the time per-epoch was comparable between the two different loaders. For 32 and 64 nodes, the time per-epoch using the regular data loader was lower-bounded by the constant data loading cost, which was limited by the  I/O rate. In contrast, using the locality-aware data loader helped the per-epoch training time to decrease further when more nodes participated. We observe 1.9x speedup over regular data loader at 64 nodes (256 learners). The results prove that our locality-aware data loading method works well in practice. 

\begin{figure}[t]
\centering
\includegraphics[width=.40\textwidth]{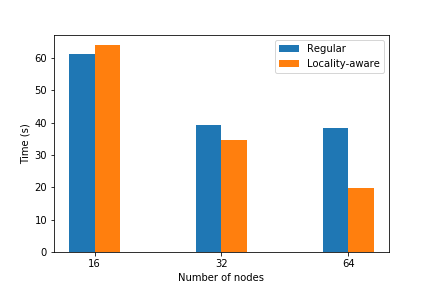}
\caption{Average epoch time of Imagenet-1K ResNet50 training in different number of nodes.}
\label{fig:imagenet-avg-epoch-time}
\end{figure}

\section{Related Work}\label{sec:related-work}

% Imagenet work
Several papers ~\cite{You_2018,jia2018highly,ying2018image,mikami2018imagenetresnet50} addressed the topic of optimizing Imagenet classification with ResNet50 on distributed systems. They mostly aimed to get a similar validation accuracy (75\% after 90 epochs) to Goyal's work~\cite{goyal2017accurate} while improving the total training time. Various novel methods that improve GPU computation time and synchronization time were proposed, but they often omit the data loading problem. 

% Data loading work
There have been mentions of the data loading problem and attempts to solve it on very large-scale deployments. In DeepIO~\cite{zhu2018entropy}, data servers store subsets of the training dataset in an in-memory cache and prioritize reuse of data from the in-memory cache. While this reduces the accesses to the storage system, the mechanism can change the mini-batch sequences and impact the model accuracy. In comparison, our method does not change the predefined mini-batch sequences. In~\cite{kurth2018exascale}, distributed caching successfully scaled the application to 4,560 nodes. It relies on the high-speed interconnect among compute nodes to reduce the data loading from the storage system and lower the I/O cost, but the total volume of data movement among compute nodes remains high. Our locality aware method complements distributed caching by reducing the data movement to a small fraction, which can make applications scale to even larger systems.

\section{Conclusion}\label{sec:conclusion}

Efficient data loading is fundamental for distributed DNN training to scale to large-scale HPC systems. We investigated the issues of the existing data loader design and proposed performance optimizations. We also identified that, by both performance modeling and empirical results, the inability to load data faster limits the scalability of distributed mini-batch SGD. Our locality-aware data loading method utilizes caches to potentially eliminate the data loading from the storage system after the first epoch, and also reduces the total data loading volume to a tiny fraction of the input dataset size. Thus, the method  lowers the bandwidth requirement effectively and makes distributed DNN training much more scalable. Our experiments show that with the proposed optimizations, we can speed up the data loading of 1,024 learners to 34x, 55x and 60x for Imagenet-1K, UCF101-RGB and UCF101-FLOW, respectively. We can also get 120x speedup for loading a 892 GB MuMMI dataset using 512 learners. Applying the optimizations to the practical Imagenet-1K classification task also shows that simply using our data loader granted $\approx$ 2x speedup with 1,024 learners while gaining comparable validation accuracy results. 

Our prototype implementation is based on PyTorch. We plan to develop a general software package of the optimized data loader  that can be used with any machine learning frameworks. 
We also plan to study the feasibility of applying our methods to other machine learning optimization methods.
We also want to explore using SSD which provides ample space and fast access, and is ideal for a hierarchical caching design.

\section*{Acknowledgment}
This work was supported under CORAL NRE Contract B604142.

\bibliographystyle{IEEEtran}
\bibliography{main}

\end{document}